\documentclass{article} 
\usepackage{geometry, amsmath}
\usepackage{amsmath,amsfonts,bm, amsthm}

\def\Figref#1{Figure~\ref{#1}}

\def\1{\bm{1}}

\def\rvx{{\mathbf{x}}}
\def\rvy{{\mathbf{y}}}

\def\vu{{\bm{u}}}
\def\vv{{\bm{v}}}
\def\vw{{\bm{w}}}
\def\vx{{\bm{x}}}
\def\vy{{\bm{y}}}
\def\vz{{\bm{z}}}

\DeclareMathAlphabet{\mathsfit}{\encodingdefault}{\sfdefault}{m}{sl}
\SetMathAlphabet{\mathsfit}{bold}{\encodingdefault}{\sfdefault}{bx}{n}

\newcommand{\E}{\mathbb{E}}

\newcommand{\R}{\mathbb{R}}

\usepackage{hyperref}
\usepackage{url}

\usepackage{graphicx}
\usepackage{subcaption}

\newcommand{\Ji}{{J_i}}

\newtheorem{theorem}{Theorem}
\newtheorem{lemma}{Lemma}
\newtheorem{corollary}{Corollary}
\newtheorem*{theorem*}{Theorem}
\newtheorem*{lemma*}{Lemma}
\newtheorem*{corollary*}{Corollary}

\newtheorem*{remark*}{Remark}
\newtheorem*{definition*}{Definition}
\newtheorem{definition}{Definition}

\title{Trajectory growth lower bounds for random sparse deep ReLU networks}

\author{Ilan Price \& Jared Tanner \thanks{Preprint }\\
Mathematical Institute\\
University of Oxford\\
\texttt{\{ilan.price, tanner\}@maths.ox.ac.uk}
}
\date{}

\begin{document}
\sloppy

\maketitle

\begin{abstract}
This paper considers the growth in the length of one-dimensional trajectories as  they are passed through deep ReLU neural networks,  which, among other  things, is  one measure of the expressivity  of  deep networks. We generalise existing  results, providing an alternative, simpler method for lower bounding expected trajectory growth through random networks, for a more general class of weights distributions, including sparsely connected  networks. We illustrate this approach by deriving bounds for sparse-Gaussian, sparse-uniform, and sparse-discrete-valued random nets. We prove that trajectory growth  can remain exponential in depth with these new distributions, including their sparse variants, with the sparsity parameter appearing in the base of the exponent.
\end{abstract}

\section{Introduction}
Deep neural networks continue to set new benchmarks for machine learning  accuracy across a wide range of tasks, and are the basis for many algorithms we use routinely and on a daily basis. One fundamental set of theoretical questions concerning deep networks relates to their \textit{expressivity}. There remain different approaches to understanding and quantifying neural network expressivity. Some results take a classical approximation theory approach, focusing on the relationship between the architecture of the network and the classes of functions it can accurately approximate (\cite{lu2017expressive,cybenko1992approximation,hornik1989multilayer}).  Another more recent approach has been to apply persistent homology to characterise expressivity (\cite{guss2018characterizing}), while \cite{Poole} focus on global curvature, and the ability of deep networks to disentangle manifolds. Other works concentrate specifically on networks with piecewise linear activation functions, using the number of linear regions (\cite{montufar2014number}) or the volume of the boundaries between linear regions (\cite{hanin2019complexity}) in input space. In 2017, \cite{raghu2017expressive} proposed trajectory length as a measure of expressivity; in particular, they consider the expected change in length of  a  one-dimensional trajectory as it is passed through Gaussian random neural networks (see \Figref{fig: illustration} for an illustration). Their primary theoretical result was that, in expectation, the length of a one-dimensional trajectory which is passed through a fully-connected, Gaussian network is \textit{lower bounded} by a factor that is exponential with depth, but not with width. 

\begin{figure}[!h]
    \centering
\begin{subfigure}[b]{0.3\textwidth}
\includegraphics[width = \textwidth]{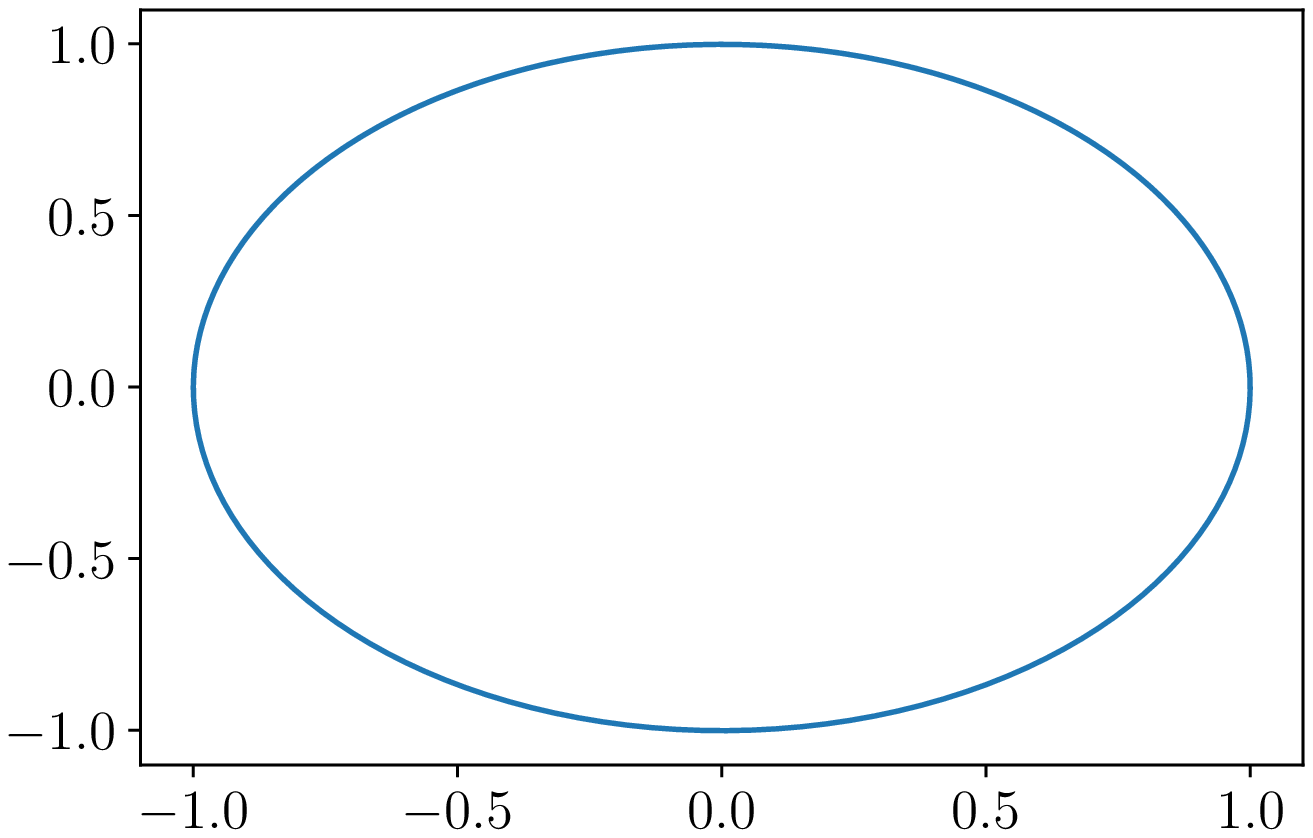}
\caption{Input}
\label{fig: input traj}
\end{subfigure}
\begin{subfigure}[b]{0.3\textwidth}
\includegraphics[width = \textwidth]{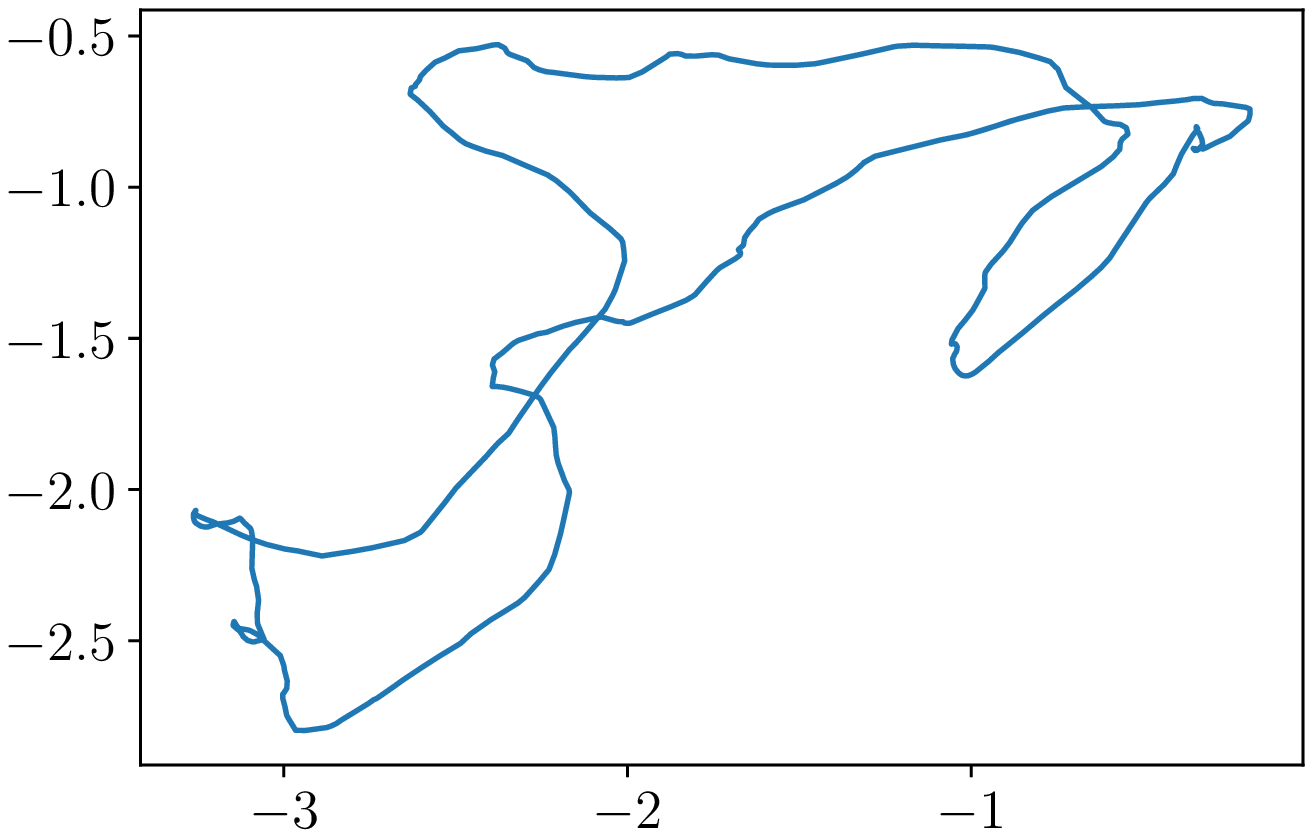}
\caption{Layer 6}
\label{fig: layer 6 traj}
\end{subfigure}
\begin{subfigure}[b]{0.3\textwidth}
\includegraphics[width = \textwidth]{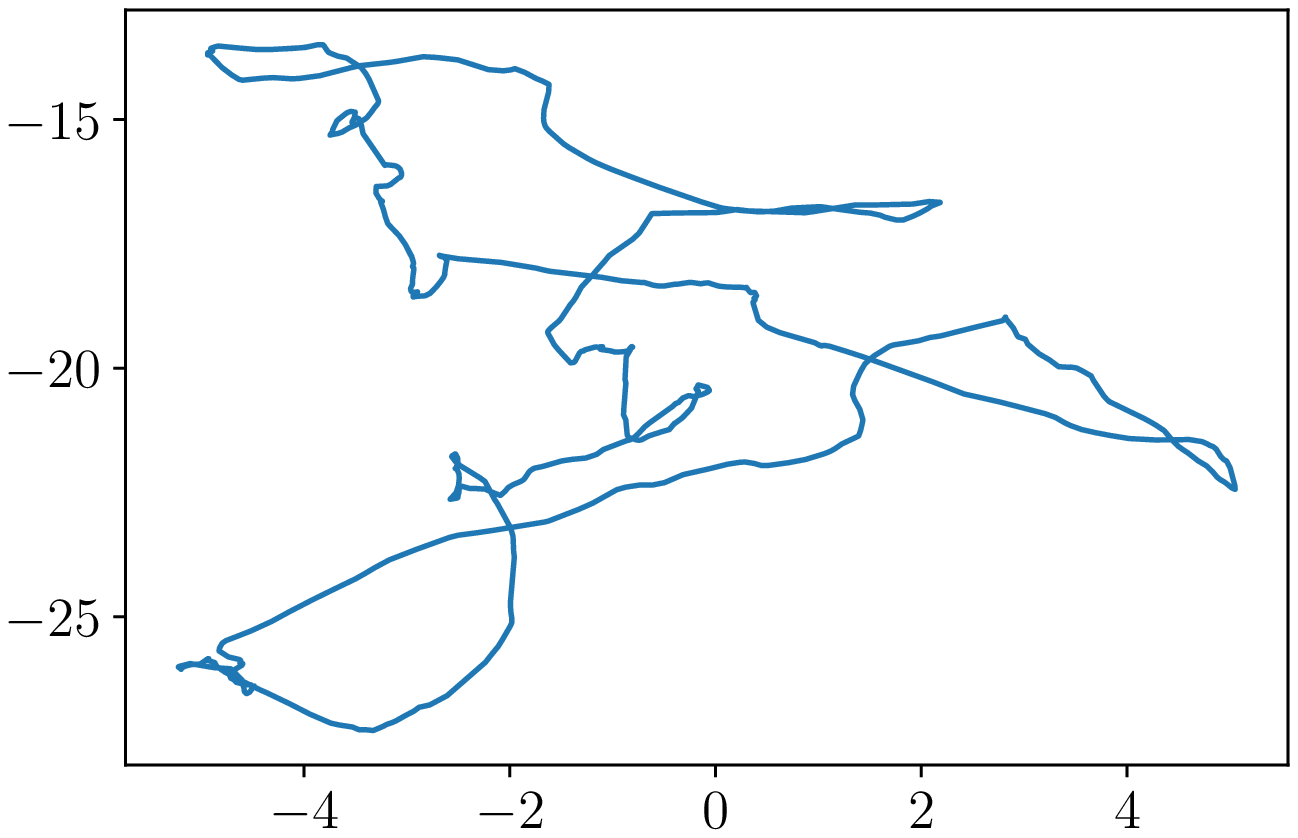}
\caption{Layer 12}
\label{fig: layer 12 traj}
\end{subfigure}
\caption{A circular trajectory, passed through a ReLU network with $\sigma_w = 2$. The plots show the pre-activation trajectory at different layers projected down onto 2 dimensions.}
\label{fig: illustration}
\end{figure}

One-dimensional trajectories and their evolution through deep networks are also of interest in their own right because they constitute simple data manifolds. Firstly, we commonly assume that the real data which we aim to correctly classify or predict with a deep network lie on one or more manifolds, and thus design a network to perform appropriately on such a manifold. Secondly, researchers are beginning to consider whether the \textit{output} (manifolds) of generator networks could be a good model for real word data manifolds, for example, as priors for a variety of inverse problems (\cite{manoel2017multi,huang2018provably}). Both of these hypotheses motivate an understanding of how manifolds are acted upon by deep networks.

Our results in this paper pertain specifically to the `trajectory length' measure of expressivity. We produce a simpler proof than in the pioneering work of \cite{raghu2017expressive}, which also generalises their results, deriving similar lower bounds for a broader class of random deep neural networks.

Theoretical work of this nature is important because it allows for more straightforward transfer and adaptation of prior theoretical results to new contexts of interest. For example, there is a current surge in research around low-memory networks, training sparse networks, and network pruning. Sparsely connected networks have shown the capacity to retain very high test accuracy (\cite{frankle2018the,han2015learning}), increased robustness (\cite{ahmad2019can,aghasi2017net}), with much smaller memory footprints, and less power consumption (\cite{yang2018energyconstrained}). The approach we take in this work enables us to extend results from dense random networks to sparse ones. It also allows us to consider the other weight distributions of sparse-Gaussian, sparse-uniform and sparse-discrete networks (see Definitions \ref{def: sparse Gaussian} - \ref{def: sparse discrete}).

More specifically we make the following contributions:

\textbf{Contributions:}
\begin{enumerate}
	\item We provide an alternative, simpler method for lower bounding expected trajectory growth through random networks, for a more general class of weights distributions (Theorem \ref{thm: general sparse}).
	\item We illustrate this approach by deriving bounds for sparse-Gaussian, sparse-uniform, and sparse-discrete random nets. We prove that trajectory growth can be exponential in depth with these distributions, with the sparsity appearing in the base of the exponential (Corollaries \ref{thm: sparse} - \ref{thm: discrete}).
	\item We observe that the expected length growth factor is strikingly similar across the aforementioned three distributions. This suggests a universality of the expected growth in length for iid centered distributions determined only by the variance and sparsity (\Figref{fig: simulations}).
\end{enumerate}

\subsection{Notation}\label{subsection: notation}

We consider feedforward ReLU deep neural networks. We denote a the $d$-th post-activation layer as $z^{(d)}$, and the subsequent pre-activation layer as $h^{(d)}$, such that
\begin{align*}
h^{(d)} = W^{(d)}z^{(d)} + b^{(d)}, \qquad z^{(d+1)} = \phi(h^{(d)}),
\end{align*} 
where $\phi(x):=\max(x,0)$ is applied elementwise. We denote $x=z^{(0)}$. 

We use $f_{NN}(x;\mathcal{P}, \mathcal{Q})$ to denote a random feedforward deep neural network which takes as input the vector $x$, and is parameterised by random weight matrices $W^{(d)}$ with entries sampled iid from the distribution $\mathcal{P}$, and bias vectors $b^{(d)}$ with entries drawn iid from distribution $\mathcal{Q}$.

\begin{definition}
A \textbf{random sparse network} with sparsity parameter $\alpha$, denoted $f_{NN}(x;\alpha, \mathcal{P}, \mathcal{Q})$, is a random feedforward network in which all \textit{weights} are sampled from a mixture distribution of the form
\begin{align*}
w_{ij} \sim \alpha \mathcal{P} + (1-\alpha) \delta,
\end{align*} 
where $\delta$ is the delta distribution at 0, and $\mathcal{P}$ is some other distribution. In other words, weights are 0 with probability $1-\alpha$, and sampled from $\mathcal{P}$ with probability $\alpha$. Biases are drawn iid from $\mathcal{Q}$.
\end{definition}

\begin{definition}\label{def: sparse Gaussian}
A \textbf{sparse-Gaussian network} is a random sparse network $f_{NN}(x;\alpha, \mathcal{P}, \mathcal{Q})$, where $\mathcal{P}~=~\mathcal{N}(0,\sigma_w^2)$ and $\mathcal{Q} = \mathcal{N}(0,\sigma_b^2)$.
\end{definition}

\begin{definition}\label{def: sparse uniform}
A \textbf{sparse-uniform network} is a random sparse network $f_{NN}(x;\alpha, \mathcal{P}, \mathcal{Q})$, where $\mathcal{P}~=~\mathcal{U}(-C_w,C_w)$ and $\mathcal{Q} = \mathcal{U}(-C_b,C_b)$.
\end{definition}

\begin{definition}\label{def: sparse discrete}
A \textbf{sparse-discrete network} is a random sparse network $f_{NN}(x;\alpha, \mathcal{P}, \mathcal{Q})$, where $\mathcal{P}$ is a uniform distribution over a finite, discrete, symmetric set $\mathcal{W}$, with cardinality $|\mathcal{W}| = N_w$, and $\mathcal{Q}$ is a uniform distribution over a finite, discrete, symmetric set $\mathcal{B}$, with cardinality $|\mathcal{B}|=N_b$.
\end{definition}
For a weight matrix $W$ in a random sparse network, with $w_i$ denoting the $i^{\text{th}}$ row, we define $w_{\mathcal{P}_i}$ as the vector containing only the $\mathcal{P}$-distributed entries of $w_i$.

We define a trajectory $x(t)$ in input space as a curve between two points, say $x_0$ and $x_1$, parameterized by a scalar $t \in [0, 1]$, with $x(0) = x_0$ and $x(1) = x_1$, and we define $z^{(d)} (x(t))= z^{(d)}(t)$ to be the image of the trajectory in layer $d$ of the network. The trajectory length $l(x(t))$ is given by the standard arc length,
\begin{align*}
\int_t \bigg \vert \bigg \vert \frac{dx(t)}{dt}\bigg \vert \bigg \vert dt.
\end{align*}
As in the work by \cite{raghu2017expressive}, this paper considers trajectories with $x(t+dt)$ having a non- trivial component perpendicular to $x(t)$ for all $t, dt$.

Finally, we  say a probability density or mass function $f_X(x)$ is even if $f_X(-x) = f_X(x)$ for all random vectors $x$ in the sample space.

\section{Expected Trajectory Growth Through Random Networks}\label{section: trajectory growth}

\cite{raghu2017expressive} considered ReLU and hard-tanh Gaussian networks with the standard deviation scaled by $1/\sqrt{k}$. Their result with respect to ReLU networks is captured in the following theorem.

\begin{theorem}[\cite{raghu2017expressive}] \label{thm: raghu}
Let $f_{NN}(x;\mathcal{N}(0,\sigma^2_w/k), \mathcal{N}(0,\sigma^2_b))$ be a random Gaussian deep ReLU neural network with layers of width $k$, then
 \begin{align*}
 	\E[l(z^{(d)} (t)) ] \geq \mathcal{O}\left(\frac{ \sigma_w \sqrt{k}}{\sqrt{k+1}}\right)^d \cdot l(x(t)),
 \end{align*}
 for $x(t)$ a 1-dimensional trajectory in input space.
\end{theorem}

There are, however, other network weight distributions which may be of interest.  For example, the expressivity and generative power of \textit{sparse} networks are of particular interest in the current moment, given the current interest in low-memory and low-energy networks, training sparse networks, and network pruning (\cite{frankle2018the,han2015learning, yang2018energyconstrained}). We prove that even for sparse random networks, trajectory growth can remain exponential in depth given sufficiently large initialisation scale $\sigma_w$. Scaling $\sigma_w$ by $1/\sqrt{k}$ can yield a width-independent lower bound on this growth. Moreover, a sufficiently high sparsity fraction $(1-\alpha)$ results in a lower bound which, instead of growing exponentially, shrinks exponentially to zero. This is captured by the following result.

\begin{corollary}[\textbf{Trajectory growth in deep sparse-Gaussian random networks}] \label{thm: sparse}
	Let $f_{NN}(x;\alpha, \mathcal{N}(0,\sigma^2_w), \mathcal{N}(0,\sigma^2_b))$ be a sparse-Gaussian, feedforward ReLU network as defined in Section \ref{subsection: notation}, with layers of width $k$. Then
	\begin{equation}
		\E[l(z^{(d)} (t)) ] \geq \left( \frac{\alpha \sigma_w\sqrt{k}}{\sqrt{2\pi}} \right)^d \cdot l(x(t)),
	\end{equation}
for $x(t)$ a 1-dimensional trajectory in input space.
\end{corollary}

Corollary \ref{thm: sparse} with $\alpha=1$ and $\sigma_w$ replaced by $\sigma_w / \sqrt{k}$ recovers a bound which is very similar to the prior bound by \cite{raghu2017expressive} in Theorem \ref{thm: raghu}.

Beyond Gaussian weights, we consider other distributions commonly used for initialising and analysing deep networks. Uniform distributions, for example, still constitute the default initialisations of linear network layers in both Pytorch  and Tensorflow (uniform according to $\mathcal{U}(-1/\sqrt{k}, 1/\sqrt{k})$ in the case of Pytorch, and uniform according to $\mathcal{U}(-6/\sqrt{k_{in} + k_{out}}, 6/\sqrt{k_{in} + k_{out}})$ -- a.k.a the Glorot/Xavier uniform initialization (\cite{glorot2010understanding}) -- in the case of Tensorflow). We prove an analogous lower bound for uniformly distributed weights.

\begin{corollary}[\textbf{Trajectory growth in deep sparse-uniform random networks}] \label{thm: uniform}
	Let $f_{NN}(x; \alpha, \mathcal{U}(-C_w, C_w), \mathcal{U}(-C_b, C_b))$ be a sparse-Uniform, feedforward ReLU network as defined in Section \ref{subsection: notation}, with layers of width $k$. Then
	\begin{align}
	\E[l(z^{(d)} (t)) ] \geq \left(\frac{\alpha C_w  \sqrt{k}}{4\sqrt{2}}\right)^d \cdot l(x(t)),
	\end{align}	
	for $x(t)$ a 1-dimensional trajectory in input space.
\end{corollary}

Another research direction which has gathered some momentum in recent years are quantized or discrete-valued deep neural networks (\cite{li2017training,hubara2016binarized,hubara2017quantized}), including recent work using integer valued weights (\cite{wu2018training}). This motivates consideration of discrete weight distributions, in addition to continuous ones. As an example of such, we prove a similar lower bound for networks with weights and biases uniformly sampled from finite, symmetric, discrete sets.

\begin{corollary}[\textbf{Trajectory growth in deep sparse-discrete random networks}] \label{thm: discrete}
	Let $f_{NN}(x; \alpha, \mathcal{P}, \mathcal{Q})$ be a sparse-discrete random feedforward ReLU network as defined in Section \ref{subsection: notation}, and layers of width $k$. Then
	\begin{align}
 \E[l(z^{(d)} (t)) ] \geq \left( \frac{\alpha \sqrt{k}}{2\sqrt{2}}\cdot \frac{ \sum_{w\in \mathcal{W}}|w|}{N_w}\right)^d \cdot l(x(t))
	\end{align}	
	for $x(t)$ a 1-dimensional trajectory in input space.
\end{corollary}

In all cases these lower bounds show how to choose the combination of $\sigma_w$ and $\alpha$ to guarantee (or not) exponential growth in trajectory length in expectation at initialisation. 

The main idea behind the derivation of these results is to consider how the length of a small piece of a trajectory (some $\|dz^{(d)}\|$) grows from one layer to the next ($\|dz^{(d+1)}\| = \|\phi(h^{d}(t+dt)) - \phi(h^{(d)}(t)\|$). In the context of random feedforward networks, we can consider piecewise linear activation functions as restrictions of $dh^{(d)}$ to a particular support set which is statistically dependent on $h^{(d)}$.  This approach was developed by \cite{raghu2017expressive}. The key to our proof is providing a more direct and more generally applicable way of accounting for this dependence than originally provided by \cite{raghu2017expressive}. Specifically, our approach lets us derive the following, more general result, from which Corollaries \ref{thm: sparse}, \ref{thm: uniform}, and \ref{thm: discrete} follow easily.

\begin{theorem}[\textbf{Trajectory growth in deep random sparse networks}] \label{thm: general sparse}
	Let $f_{NN}(x; \alpha, \mathcal{P}, \mathcal{Q})$ be a random sparse network as defined in Section \ref{subsection: notation}, with layers of width $k$. Let $\mathcal{P}$ and $\mathcal{Q}$ be such that the joint distribution over a vector of independent elements from both distributions is even. If $\E [|\vu^\top \hat{w}_{\mathcal{P}_i}|] \geq M\|\vu\|$ for any constant vector $\vu$, for all $i$, then
	\begin{align}\label{eqn:  general sparse}
	\E[l(z^{(d)} (t)) ] \geq \left(\frac{\alpha M \sqrt{k}}{2}\right)^d \cdot l(x(t))
	\end{align}	
	 for $x(t)$ a 1-dimensional trajectory in input space.
\end{theorem}

\begin{remark*}
It is trivial to amend this result for networks where the width, distribution, and sparsity varies layer by layer, in which case the lower bound \eqref{eqn:  general sparse} is replaced by 
\begin{align*}
    \prod_{j=i}^d \left( \frac{\alpha_j M_j \sqrt{k_j}}{2}\right) \cdot l(x(t))
\end{align*}
Moreover, the bounds from Theorem \ref{thm: general sparse} and Corollaries \ref{thm: sparse} - \ref{thm: discrete} hold true in the 0 bias case as well.
\end{remark*}

\section{Proof of Theorem \ref{thm: general sparse}}

We prove Theorem \ref{thm: general sparse} in three stages: i)  We turn the problem into one of bounding from below the change in the length of an infinitesimal  line segment; ii) we account simply and explicitly for the dependence generated by the ReLU activation;  and iii)  we break this dependence by taking advantage of the symmetry characterising this class of distributions.   Supporting lemmas can be found in Appendix A.

\begin{proof}

\textbf{Stage 1:} 

For the first stage of proof, we will closely follow \cite{raghu2017expressive}. We are interested in deriving a lower bound of the form, 
\begin{align}
\E\left[\int_t \bigg \vert \bigg \vert \frac{dz^{(d)}(t)}{dt}\bigg \vert \bigg \vert dt \right] \geq C\cdot \int_t \bigg \vert \bigg \vert \frac{dx(t)}{dt}\bigg \vert \bigg \vert dt ,
\end{align}
for some constant $C$. As noted by \cite{raghu2017expressive}, it suffices to instead derive a bound of the form 
\begin{align*}
\E\left[\|dz^{(d)}(t)\|  \right] \geq C\|dx(t)\|,
\end{align*}
since integrating over $t$ yields the desired form. Our approach will be to derive a recurrence relation between $\|dz^{(d+1)}\|$ and $\|dz^{(d)}\|$, where we refrain from explicitly including the dependence of $dz$ on $t$, for notational clarity. 

Next, like \cite{raghu2017expressive}, our proof relies on the observation that
\begin{align*}
dz^{(d+1)} &= \phi( W^{(d)}z^{(d)}(t+\delta t) + b^{(d)}) - \phi(W^{(d)}z^{(d)}(t) + b^{(d)}) \\ 
& = \phi^{(d)}(t+\delta t) - \phi^{(d)}(t)  \\
& = d\phi^{(d)},    
\end{align*} 
and that since $\phi$ is the ReLU operator, $\frac{d\phi}{dh^{(d)}_j}$ is either 0 or 1. When $z^{(d)}$ is fixed  independently of $W^{(d)}$ and $b^{(d)}$, then $P(h^{(d)}_j = 0) =0$ (see the  preamble to Lemma \ref{conditional unnecesary discrete} for more detail on this), and thus we need only note that $d\phi^{(d)}_j~=~dh^{(d)}_j$ when $h^{(d)}_j>0$, and $d\phi^{(d)}_j=0$ when $h^{(d)}_j<0$. We define $\mathcal{A}^{(d)}$  to be the  set of `active nodes' in layer $d$; specifically,
\begin{align*}
\mathcal{A}^{(d)}:= \{ j: h_j^{(d)}>0\},     
\end{align*}
 and $I_{\mathcal{A}^{(d)}}\in \R^{k \times k}$ is defined as the matrix with ones on the diagonal entries indexed by set $\mathcal{A}^{(d)}$, and 0 everywhere else. We can then write 
\begin{align*}
\| dz^{(d+1)}\| & = \|I_{\mathcal{A}^{(d)}} (h^{(d)}(t+dt) - h^{(d)}(t) ) \| \\
 & = \|I_{\mathcal{A}^{(d)}} W^{(d)} dz^{(d)} \|.
\end{align*}
From here we will drop the weight index $(d)$ to minimise clutter in the exposition.

It is at this point where we depart from the proof strategy used by \cite{raghu2017expressive}. The next steps in their proof depend heavily on the weight matrices in the network being Gaussian. For example, they require that a weight matrix after rotation has the same, i.i.d. distribution as the matrix before rotation. Instead, our proof can tackle a number of other, non-rotationally-invariant distributions, as well as sparse networks.

\textbf{Stage 2:} 

The next stage of the proof begins by noting that after conditioning on size of the set $\mathcal{A}$, 
\begin{align}\label{eqn: rid of index matrix}
\E [ \|I_\mathcal{A} W dz^{(d)}\| \ | \ |\mathcal{A}|] = \E [ \| \hat{W}dz^{(d)} \| \ | \ \hat{w}_i^\top z^{(d)} + \hat{b}_i > 0 \ \forall i, |\mathcal{A}| ],
\end{align}
where $\hat{W} \in \R^{|\mathcal{A}| \times k}$ is the matrix comprised of the rows of $W$ indexed by $\mathcal{A}$, and we denote the $i$-th row of $\hat{W}$ as $\hat{w}_i$, and the $i$-th entry of $\hat{b}$ as $\hat{b}_i$. Equation \ref{eqn: rid of index matrix} follows since the elements of $Wdz^{(d)}$ are i.i.d., and $\mathcal{A}^{(d)}$ selects all entries whose corresponding entries in $h^{(d)}$ have positive values. Thus, in expectation, pre-multiplying by the matrix $I_{\mathcal{A}^{(d)}}$ is equivalent to considering  $\hat{W}dz^{(d)}$ instead of $I_\mathcal{A} Wdz^{(d)}$ together with conditioning on the fact that every element in the vector $\hat{W}z^{(d)} + \hat{b}$  is positive.

This gives us
\begin{align}
\E [ \|I_\mathcal{A} W dz^{(d)}\| \ ] & =  \E \left[ \mathop{{}\E}_{\hat{w}_1} \mathop{{}\E}_{\hat{w}_2} \cdots \mathop{{}\E}_{\hat{w}_{|\mathcal{A}|}} \left[ \sqrt{\sum_{i=1}^{|\mathcal{A}|} ( \hat{w}_i^\top  dz^{(d)})^2} \ \bigg \vert \ \hat{w}_i^\top z^{(d)} + \hat{b}_i > 0 \ \forall i, |\mathcal{A}|  \right] \right] \label{split over rows}\\ 
& = \E \left[ \mathop{{}\E}_{\hat{w}_1} \mathop{{}\E}_{\hat{w}_2} \cdots \mathop{{}\E}_{\hat{w}_{|\mathcal{A}|}} \left[ \sqrt{\sum_{i=1}^{|\mathcal{A}|} | \hat{w}_i^\top  dz^{(d)} |^2} \ \bigg \vert \ \hat{w}_i^\top z^{(d)} + \hat{b}_i > 0 \ \forall i, |\mathcal{A}|  \right] \right] \label{trivial} \\
&\geq \E \left[ \sqrt{\sum_{i=1}^{|\mathcal{A}|} \mathop{{}\E}_{\hat{w}_i} [ | \hat{w}_i^\top  dz^{(d)} | \ | \hat{w}_i^\top z^{(d)} + \hat{b}_i > 0]^2 } \right] \label{jensen},
\end{align}
where \eqref{split over rows} follows from the analysis above and the independence of each $\hat{w}_i$, \eqref{trivial} is trivial, and \eqref{jensen} follows from iteratively applying Jensen's inequality, after noting that $f(x)=\sqrt{x^2+C}$ is convex for $x, C\geq0$.

Now let $\Ji$ denote the (random) index set of the $\mathcal{P}$-distributed entries of $\hat{w}_i$,  and let $w_\Ji, dz^{(d)}_{J_i}, z^{(d)}_{J_i}$ denote the restrictions to the indices in $\Ji$ of $\hat{w}_i$, $dz^{(d)}$ and $z^{(d)}$  respectively.  Then $\hat{w}_i^\top z^{(d)} = w_\Ji^\top z^{(d)}_\Ji$, and $  \hat{w}_i^\top  dz^{(d)} = w_\Ji ^\top dz^{(d)}_\Ji$, such that, after conditioning on $\Ji$, we have that 
\begin{align}
\E [ \| \hat{W}p \| \ | \ \hat{w}_i^\top z^{(d)} + \hat{b}_i > 0 \ \forall i, |\mathcal{A}| ] \geq \underbrace{\E \left[ \sqrt{\sum_{i=1}^{|\mathcal{A}|} \overbrace{ \mathop{{}\E}_{\Ji}\big[ \underbrace{\mathop{{}\E}_{w_\Ji} [ | w_\Ji^\top  dz^{(d)}_\Ji | \ | w_\Ji^\top z^{(d)}_\Ji  + \hat{b}_i > 0, \Ji] }_{(*)}\big] }^{(**)}}\ ^2 \right]}_{(***)}. \label{telescope of expectations}
\end{align}

\textbf{Stage 3:} 

The third stage of the proof is to work our way from the inside out, lower bounding $(*)$ first, then $(**)$, and finally $(***)$. 

Consider the expectation in $(*)$. Having conditioned on $\Ji$, we can define  $X =  w_\Ji ^\top dz^{(d)}_\Ji$ and  $Y = w_\Ji^\top z^{(d)}_\Ji  + \hat{b}_i $, such that lower bounding $(*)$ means lower bounding
\begin{align}
\E [ |X| \ | Y>0].
\end{align} 
By assumption the joint distribution over $G = [w_{\Ji, 1}, \dots, w_{\Ji, k}, \hat{b}_i]^\top$ is even. The vector $H = [X,Y, w_{\Ji, 3} \dots, w_{\Ji,k}, \hat{b}_i]^\top$ is obtained by a linear transformation of $G$ (which is invertible since $\|z^{(d)}\|$ is not parallel to $\|dz^{(d)}\|$). Thus by Lemma \ref{lemma:  continuous case - even distribution of linear transformation} (continuous) or Lemma \ref{lemma:  discrete case - even distribution of linear transformation} (discrete) this joint distribution over $H$ is also even, and by Lemma \ref{lemma: continuous case, even after marginalisation} (continuous) or Lemma \ref{lemma: discrete case, even after marginalisation} (discrete), the joint distribution of $[X,Y]^\top$ is even too. We can therefore apply Lemma \ref{lemma: conditional unnecesary} (continuous) or Lemma \ref{conditional unnecesary discrete} (discrete) and need only consider $\E[|X|]$, which is bounded as
\begin{align}
	\E[|X|] \geq M\|dz^{(d)}_\Ji\|,
\end{align}
again by assumption. 

Having bounded $(*)$, we average over $\Ji$ to get $(**)$,  for which we can apply Lemma \ref{lemma: norm of subvector} to get 
\begin{align}
\mathop{{}\E}_{\Ji} [M\|dz^{(d)}_\Ji\|] \geq \alpha M \|dz^{(d)}\|.\label{eq: lowerbound  subvector in proof}
\end{align}
Finally, we can bound $(***)$ as follows
\begin{align}
\E [\|I_\mathcal{A} W dz^{(d)}\| ] & \geq \mathop{{}\E}_{|\mathcal{A}|} \left[ \sqrt{\sum_{i=1}^{|\mathcal{A}|} \alpha^2 M^2 \|dz^{(d)}\|^2} \right] \label{substitute in}\\
& = \mathop{{}\E}_{|\mathcal{A}|} \left[\sqrt{|\mathcal{A}| \cdot \alpha^2 M^2 \|dz^{(d)}\|^2} \right]\label{independent of i}\\
& \geq \mathop{{}\E}_{|\mathcal{A}|} \left[\frac{1}{\sqrt{k} \alpha M  \|dz^{(d)}\|} \cdot |\mathcal{A}| \cdot \alpha^2 M^2 \|dz^{(d)}\|^2 \right]\label{linear bound pf sqrt}\\
& =  \frac{\alpha M \|dz^{(d)}\|}{\sqrt{k}} \cdot \E [ |\mathcal{A}| ].
\end{align}
where \eqref{substitute in} is obtained by substituting the bound for $(**)$ into the inequality in \eqref{telescope of expectations}, \eqref{independent of i} follows since there is no dependence on $i$ in the summed terms, and \eqref{linear bound pf sqrt} follows since for any $0 \leq \gamma \leq \text{max}(\gamma)$, $\sqrt{\gamma} \geq \frac{1}{\sqrt{max(\gamma)}} \gamma $, and $|\mathcal{A}|$ is at most $k$. 

The proof is concluded by calculating  $\E[ |\mathcal{A}| ]$. Since $|\mathcal{A}|$ is the number of entries in the vector $h^{(d)}$ which are positive, and each entry in that vector is an independent, centred random variable, $|\mathcal{A}|$ has a binomial distribution with probability $1/2$, and therefore an expected value of $k/2$. Plugging this in yields the final recursive relation between $\|dz^{(d+1)}\|$ and $\|dz^{(d)}\|$,
\begin{align*}
\E [\|dz^{(d+1)} \| ] \geq \frac{\alpha M \sqrt{k}}{2} \|dz^{(d)}\| .
\end{align*}
Iterative application of this result starting at the first layer yields the final result. 

\end{proof}

Let us illustrate the ease with which Corollaries \ref{thm: sparse}, \ref{thm: uniform} and \ref{thm: discrete} are obtained. In the case of each distribution, we need to do two things. First, we must verify that the necessary assumption holds in the case of those distributions $\mathcal{P}$ and $\mathcal{Q}$: that the joint distribution over a vector of independent elements from both distributions is even. Second, we must derive a bound of the form $\E[|\vu^\top \vw | ] \geq M \|\vu\|$, where $w_i \sim \mathcal{P}$, and substitute $M$ into Theorem \ref{thm: general sparse}.

When $\mathcal{P}$ and $\mathcal{Q}$ are centred Gaussians, the joint distribution over elements from one or both distributions is a multivariate Gaussian, with an even joint probability density function.  Moreover,  for $U=\vu^\top \vw $, $\E[|U|]$ has a closed form solution, 
\begin{align*}
\E[|U|] = \frac{\sqrt{2} \sigma_w}{\sqrt{\pi}}  \|\vu\|
\end{align*}

When $\mathcal{P}$ and $\mathcal{Q}$ are centred uniform distributions, the joint distribution is uniform over the polygon bounded in each dimension by the symmetric bounds $[-C_w,C_w]$ or $[-C_b,C_b]$, and thus is even.  Next, to bound $\E[|U|]$, we apply the Marcinkiewicz-Zygmund inequality with $p=1$, using the optimal $A_1$ from Lemmas \ref{MZ_inequality} and \ref{optimal_MZ_constants}, to get that 
\begin{align*}
\E[|U|] & \geq \frac{C_w}{2\sqrt{2}} \|\vu \|;
\end{align*}
for details of this derivation, see Lemma \ref{bound for mod x uniform}. 

Likewise, when $\mathcal{P}$ and $\mathcal{Q}$ are uniform distributions over discrete, symmetric, finite sets $\mathcal{W}$ and $\mathcal{B}$ respectively, we make a discrete analogue of the argument made in the continuous uniform case to confirm the necessary assumption holds. Bounding $\E[|U|]$ in this case also follows from a very similar argument to that made in the continuous case, detailed in full in Lemma \ref{lemma: bound for mod x uniform discrete}, yielding
\begin{align*}
\E[|U|] \geq \frac{\sum_{w\in \mathcal{W}}|w|}{\sqrt{2} N_w} \|\vu\|.
\end{align*}

\section{Numerical Simulations}

In this section we demonstrate, through numerical simulations, how the relationships between the the network's distributional and architectural properties observed in practice compare with those described in the lower bounds of Corollaries \ref{thm: sparse} - \ref{thm: discrete}. To this end, we use as our trajectory a straight line between two (normalised) MNIST datapoints\footnote{In this experiment we chose the $101^{\text{st}}$ and $1001^{\text{st}}$ points from the MNIST test set, but the choice of points does not qualitatively change the results.}, discretized into 10000 pieces. For each combination of distribution and parameters, we pass the aforementioned line through 100 different deep neural networks of width $784$, and average the results. Specifically, we consider three different networks types, sparse-Gaussian, sparse-uniform, and sparse-discrete networks, from Definitions \ref{def: sparse Gaussian} - \ref{def: sparse discrete} respectively. For each distribution we consider different values of network fractional density $\alpha$ ranging from $0.1$ to $1$. In the sparse-Gaussian networks, non-zero weights are sampled from $\mathcal{N}(0,\sigma^2_w / k)$, and biases from $\mathcal{N}(0, 0.01^2)$. In the sparse-Uniform networks, non-zero weights are sampled from $\mathcal{U}(-C/\sqrt{k},C/\sqrt{k})$, and biases from $\mathcal{U}(-0.01,0.01)$. In the sparse-discrete networks, non-zero weights are uniformly sampled from $\mathcal{W}:= (1/\sqrt{k}) \odot \{-C, -(C+1), \dots, C-1, C \}$, and biases from $\mathcal{B}:= \{-0.01,0.01\}$. We do this for a variety of $\sigma_w$ and $C$ values. The results are shown in Figures \ref{fig: d_dep} and \ref{fig: simulations}. 

\begin{figure}[!h]
\centering
\includegraphics[width = 0.4\textwidth]{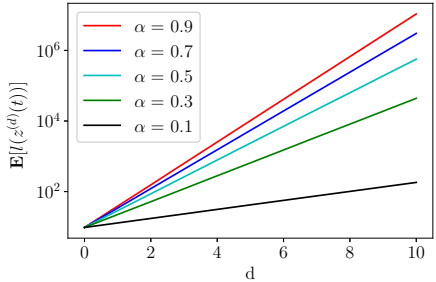}
\caption{Expected length of a line connecting two MNIST data points as it passes through a sparse-Gaussian deep network, plotted at each layer $d$.}
\label{fig: d_dep}
\end{figure}

Figure \ref{fig: d_dep} plots the average length of the trajectory at layer $d$ of a sparse-Gaussian network, with $\sigma_w  = 6$ and for different choices of sparsity ranging from $0.1$ to $0.9$. We see exponential increase of expected length with depth even in sparse networks, with smaller slopes for smaller $\alpha$ (higher sparsity). 
\begin{figure}[!h]
    \centering
\begin{subfigure}[b]{0.45\textwidth}
\includegraphics[width = \textwidth]{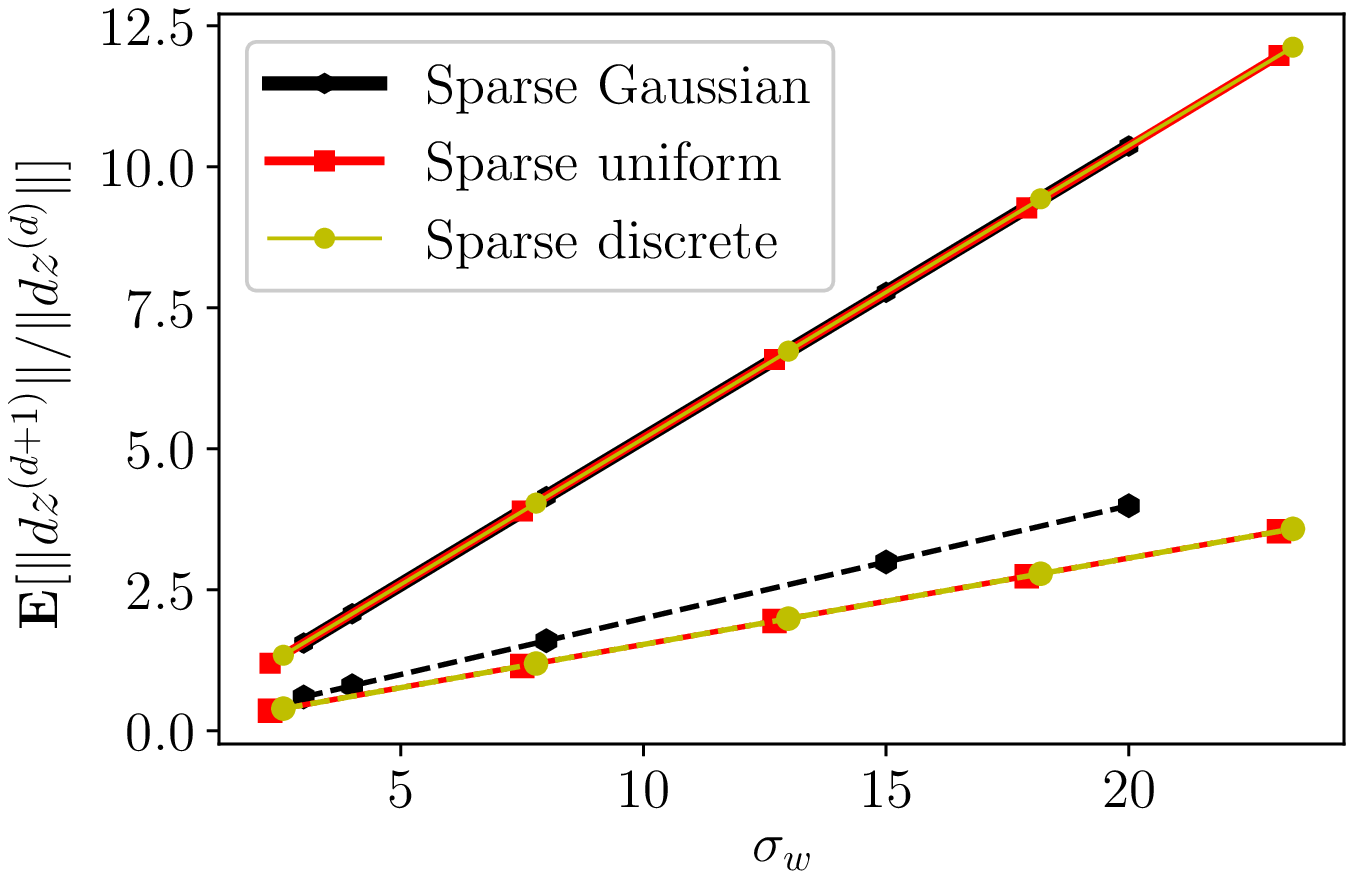}
\caption{}
\label{fig: sig_dep}
\end{subfigure}
\begin{subfigure}[b]{0.45\textwidth}
\includegraphics[width = \textwidth]{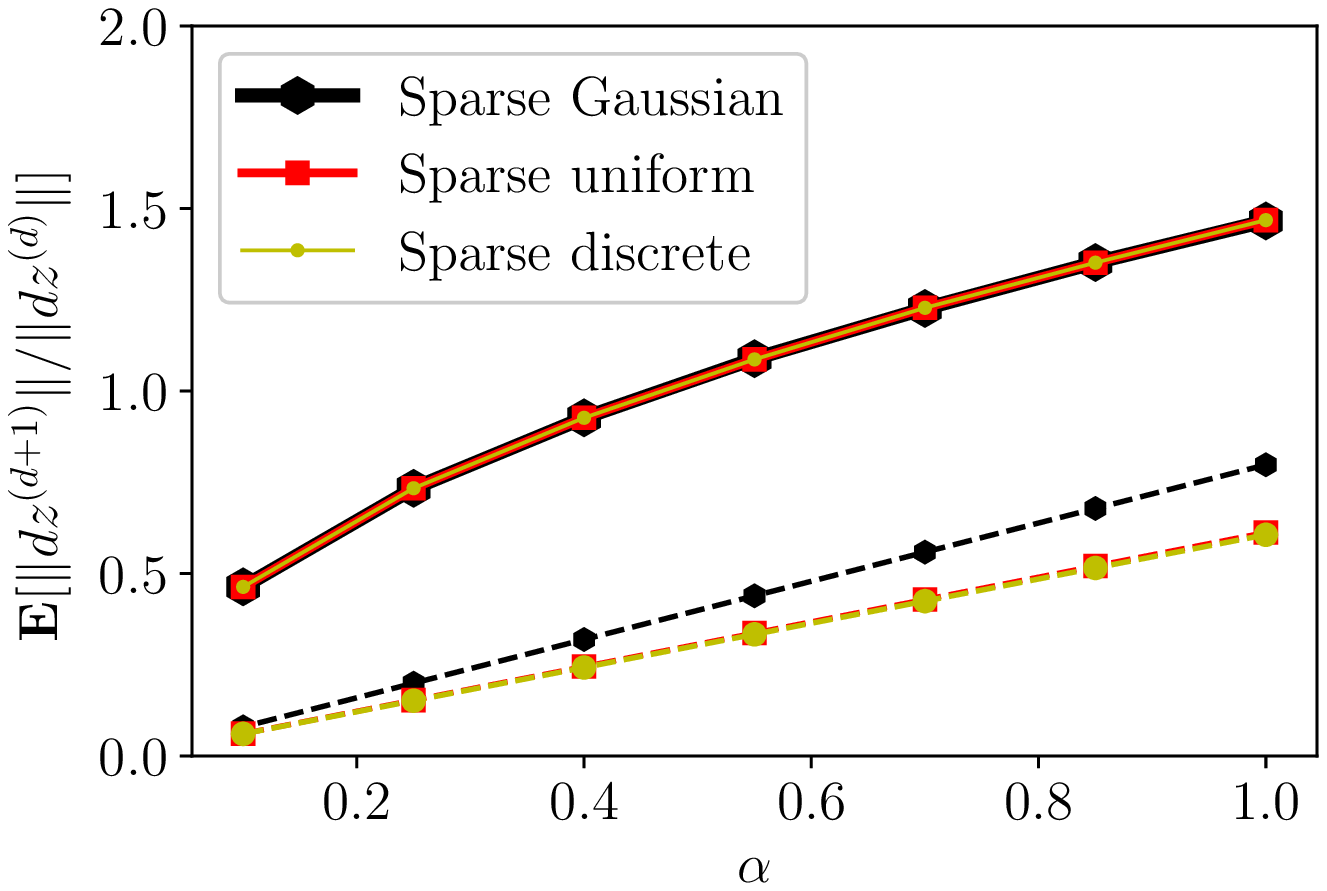}
\caption{}
\label{fig: alpha_dep}
\end{subfigure}
\caption{Expected growth factor, that is, the expected ratio of the length of any very small line segment in layer $d+1$ to its length in layer  $d$. \Figref{fig: sig_dep} shows the dependence on the variance of the weights' distribution, and \Figref{fig: alpha_dep} shows the dependence on sparsity.}
\label{fig: simulations}
\end{figure}
In Figures \ref{fig: sig_dep} and \ref{fig: alpha_dep} we plot the growth ratio of a small piece of the trajectory from one layer to the next, averaged over all pieces, at all layers, and across all 100 networks for a given distribution. This $\E[\|dz^{(d+1)}\|/\|dz^{(d)}\|]$ corresponds to the base of the exponential in our lower bound. The solid lines reflect the observed averages of this ratio, while the dashed lines reflect the lower bound from Corollaries \ref{thm: sparse}, \ref{thm: uniform}, and \ref{thm: discrete}. Figure \ref{fig: sig_dep} illustrates the dependence on the standard deviation of the respective distributions (before scaling by $1/\sqrt{k}$), with $\alpha$ fixed at $\alpha=0.5$. We observe both that the lower bounds clearly hold, and that the dependence on $\sigma_w$ is linear in practice, exactly as we expect from our lower bounds. Figure \ref{fig: alpha_dep} shows the dependence of this ratio on the sparsity parameter $\alpha$, where we have fixed $\sigma_w = 2$ for all distributions. Once again, the lower bounds hold, but in this case there is a slight curve in the observed values, not a strictly linear relationship. The reason for this is that the linear bound we provide is necessary in order to account for the more pathological cases of $dz$. This is discussed in more depth in Appendix \ref{appendix: alpha dep}. 

One striking observation in Figures \ref{fig: sig_dep} and \ref{fig: alpha_dep} is that for a given $\sigma_w$, the observed $\E[\|dz^{(d+1)}\|/\|dz^{(d)}\|]$ matches perfectly across all three distributions, for different values of $\sigma_w$ and different $\alpha$. This remains true when we repeat the experiments with different datapoints, and with points chosen uniformly at random in a high-dimensional space, both when the trajectory considered is a straight line and when it is not (e.g. arcs in two or more dimensions.) See Appendix \ref{appendix: additional experiments} for these figures. Another implication of these experiments is that they give some guidance for how to trade off weight scale against sparsity depending on the desired network properties. For example, Figure 3b considers the initialisation scheme with $\sigma_w = 2/\sqrt{k}$. We see that the empirically observed growth factor from one layer to the next is approximately 1.5 when the matrices are dense ($\alpha=1$), while the growth factor is 1 with $\alpha \approx 0.5$, and less than one as $\alpha$ decreases further. 

\section{Conclusion}

Our proof strategy and results generalise and extend previous work by \cite{raghu2017expressive} to develop theoretical guarantees lower bounding expected trajectory growth through deep neural networks for a broader class of network weight distributions and the setting of sparse networks. We illustrate this approach with Gaussian, uniform, and discrete valued random weight matrices with any sparsity level.

\bibliography{biblio}
\bibliographystyle{plain}
\newpage
\appendix
\section{Supporting Lemmas}

\begin{lemma}\label{lemma:  continuous case - even distribution of linear transformation}
    Let $f_X(\rvx)$ be an even joint probability density function over random vector $X \in \R^k$. Let $A\in \R^{k \times k}$ be an invertable linear transformation such that $Y = AX$. Then the joint density $f_Y (\rvy)$ is also even.
\end{lemma}
\begin{proof}
Wlog we assume $f_X $ is defined on $\R^k$.  To calculate the density over $Y \in \R^k$ we make a change of variables such that
\begin{align}
    f_Y (\rvy) = f_X(A^{-1}\rvy)|A^{-1}|.
\end{align}
Since $A$ is one-to-one, we have that $f_X(\vx) = f_X(A^{-1}\vy)$ for some $\vy$, and $f_X$ is even, so $f_X (A^{-1}\vy) = f_X (-(A^{-1}\vy)) = f_X (A^{-1}(-\vy))$ for all $\vy$. Putting this together completes the proof,
\begin{align}
    f_Y (\rvy) = f_X(A^{-1}\rvy)|A^{-1}| = f_X(A^{-1}(-\rvy))|A^{-1}| = f_Y (-\rvy)
\end{align}
\end{proof}

\begin{lemma}\label{lemma:  discrete case - even distribution of linear transformation}
    Let $f_X(\rvx)$ be an even joint probability mass function over random vector $X \in \R^k$. Let $A\in \R^{k \times k}$ be an invertable linear transformation such that $Y = AX$. Then the joint mass function $f_Y (\rvy)$ is also even.
\end{lemma}
\begin{proof}
$f_X $ is defined on some discrete, finite, symmetric set $\mathcal{X}$.  To calculate the density over $Y \in \mathcal{Y}:=\{Ap: p \in \mathcal{X}\}$ we make a change of variables such that
\begin{align}
    f_Y (\rvy) = \sum_{\vx \in \{A\vx = \vy \}} f_X(\rvx).
\end{align}
Since $A$ is one-to-one, we have that $f_X(\vx) = f_X(A^{-1}\vy)$ for some $\vy$, and $f_X$ is even, so $f_X (A^{-1}\vy) = f_X (-(A^{-1}\vy)) = f_X (A^{-1}(-\vy))$ for all $\vy$. Putting this together completes the proof,
\begin{align}
    f_Y (\rvy) = \sum_{\vx \in \{A\vx = \vy \}} f_X(A^{-1}\vy) = \sum_{\vx \in \{A\vx = \vy \}} f_X(A^{-1}(-\vy)) = f_Y (-\rvy)
\end{align}
\end{proof}

\begin{lemma}\label{lemma: continuous case, even after marginalisation}
	Let $f_{X_1, \dots, X_k}(x_1, \dots, x_k)$ be an even probability density function. Then $f_{X_1, \dots, X_{k-1}}(x_1, \dots, x_{k-1}) =  \int_{-\infty}^{\infty} f_{X_1, \dots, X_k}(x_1, \dots, x_k)  dx_k$ is also even. 
\end{lemma}
\begin{proof}
	\begin{align*}
		f_{X_1, \dots, X_{k-1}}(x_1, \dots, x_{k-1}) &= \int_{-\infty}^{\infty} f_{X_1, \dots, X_k}(x_1, \dots, x_k) dx_k \\
		& =\int_{-\infty}^{\infty} f_{X_1, \dots, X_k}(-x_1, \dots, -x_k) dx_k\\
		& = \int_{-\infty}^{\infty} f_{X_1, \dots, X_k}(-x_1, \dots, -x_{k-1}, x_k) dx_k\\
		& = f_{X_1, \dots, X_{k-1}}(-x_1, \dots, -x_{k-1}) 
	\end{align*}
The first and last equalities follow from the definition of marginalisation of random variables. The second equality follows from the assumption that $f_{X_1, \dots, X_k}$ is even, and the third equality follows from the change of variables: $-x_k \longrightarrow x_k$. 
\end{proof}

\begin{lemma}\label{lemma: discrete case, even after marginalisation}
	Let $X_1,\dots,X_k$ be discrete random variables with symmetric support sets $\mathcal{X}_1,\dots, \mathcal{X}_k$ respectively, i.e. $x_i \in \mathcal{X}_j \iff -x_i \in \mathcal{X}_j$. Let $P(X_1=x_1, \dots, X_k=x_k)$ be an even probability mass function such that $P(X_1=x_1, \dots, X_k=x_k)=P(X_1=-x_1, \dots, X_k=-x_k)$ .
	
	Then $P(X_1=x_1, \dots, X_{k-1}=x_{k-1})$ is also even. 
\end{lemma}
\begin{proof}
	
	\begin{align}
	P(X_1=x_1, \dots, X_{k-1}=x_{k-1}) &= \sum_{x_k \in \mathcal{X}_k} P(X_1=x_1, \dots, X_{k}=x_{k}) \label{discrete marginal 1}\\
	& = \sum_{x_k \in \mathcal{X}_k} P(X_1=-x_1, \dots, X_k=-x_k) \label{even by assumption}\\
	& = \sum_{-x_k \in \mathcal{X}_k} P(X_1=-x_1, \dots, X_k=x_k) \label{change of vars}\\
	& = \sum_{x_k \in \mathcal{X}_k} P(X_1=-x_1, \dots, X_k=x_k) \label{same support set}\\
	& = P(X_1=-x_1, \dots, X_{k-1}=-x_{k-1}) \label{discrete marginal 2}
	\end{align}
	Lines \ref{discrete marginal 1} and \ref{discrete marginal 2} follow from the definition of marginal distributions, \eqref{even by assumption} follows by assumption, \eqref{change of vars} follows fro a change of variables, and \eqref{same support set} follows since summing over $-x_k$ is equivalent to summing over $x_k$.
	
\end{proof}

\begin{lemma} \label{lemma: conditional unnecesary}
	Let $X$ and $Y$ be random variables with an even joint probability density function $f_{XY}(x,y)$. Then
	\begin{align*}
	\E[|X| \ | \ Y>0] = \E [|X|]
	\end{align*}
\end{lemma}

\begin{proof}
	Letting $|X|=Z$, we can make a straightforward change of variables to calculate the joint distribution $f_{ZY}(z,y)$, which works out to be 
	\begin{align*}
	f_{ZY}(z,y) = f_{XY}(z,y) + f_{XY}(-z, y)
	\end{align*}
	for $z\geq0$ and $y\in \R$. Then we have that
	\begin{align*}
	\E [Z| Y>0]& = \int_0^\infty z \cdot f_{Z|Y>0}(z|y>0) dz \\
	& = \int_0^\infty z \cdot \frac{f_{Z,Y>0}(z,y>0)}{\int_0^\infty f_{Y}(y)dy} dz \\
	& = 2 \int_0^\infty z \cdot f_{Z,Y>0}(z,y>0) dz \\
	& = 2 \int_0^\infty z \int_0^\infty f_{ZY}(z,y) dy dz\\
	& = 2 \int_0^\infty z \int_0^\infty (f_{XY}(z,y) +  f_{XY}(-z,y)) dy dz.
	\end{align*}
	One the other hand, we have that 
	\begin{align*}
	\E [Z]& = \int_0^\infty z \cdot f_Z(z) dz \\
	& = \int_0^\infty z \cdot (f_X(z) + f_X(-z)) dz \\
	& = 2 \int_0^\infty z \cdot f_X(z) dz \\
	& = 2 \int_0^\infty z \cdot \int_{-\infty}^{\infty}f_{XY}(z,y) dy dz \\
	& = 2 \int_0^\infty z \cdot \left( \int_{-\infty}^{0}f_{XY}(z,y) dy +  \int_{0}^{\infty}f_{XY}(z,y) dy \right) dz \\
	\end{align*}
	Comparing the expressions for $\E [Z| Y>0]$ and $\E[Z]$, we can see that they are equal if
	\begin{align*}
	\int_{-\infty}^{0}f_{XY}(z,y) dy =  \int_0^\infty f_{XY}(-z,y)dy.
	\end{align*}
	A change of variables on the left hand side from $y$ to $-y$ yields
	\begin{align*}
	\int_{-\infty}^{0}f_{XY}(z,y) dy =  \int_0^\infty f_{XY}(z,-y)dy.
	\end{align*}
	and by assumption, we know that $f_{XY}(z,-y) = f_{XY}(-z,y)$ since $f_{XY}$ is even, which completes the proof.
	
\end{proof}

Lemma \ref{lemma: conditional unnecesary} implicitly makes use of the fact that $P(Y=0) =0$, which follows from $w_\Ji$ and $\hat{b}_i$ being continuous random variables, and $Y=w_\Ji^\top z_\Ji + \hat{b}_i$, with $z_\Ji$ being fixed independent of $w_\Ji$.  We similarly make use of the fact that $P(Y=0)=0$ in the application of Lemma \ref{conditional unnecesary discrete}, though that this is true is less immediately apparent in the discrete case. For clarity, let us define $\vv:= [w_\Ji, \hat{b}_i]$, the concatenation of $w_\Ji$ and $\hat{b}_i$, and $\hat{\vz}:= [z_\Ji, 1]$, the concatenation of $z_\Ji$ and $1$, such that $Y = \vv^\top \hat{\vz}$.  Associated with the discrete distribution over $\vv$ there are $N_w^{|\Ji|}N_b$ possible discrete random vectors in $\R^{|\Ji|+1}$.  The set of vectors $\hat{\vz} \in \R^{|\Ji|+1}$ orthogonal to such a discrete set is measure zero, and  as such for $\hat{\vz}$ fixed independent of the choice of the discrete measure $\vv$ we have $P(\vv^\top \hat{\vz}=0)=0$.  If however $\hat{\vz}$ were selected with knowledge of  the discrete distribution $\vv$ then one of two cases will occur; either $\vv^\top \hat{\vz}\neq 0$, or $\hat{\vz}$ is selected to be from the measure zero set of vectors orthogonal to any of the $N_w^{|\Ji|}N_b$ vectors generated by $\vv$. In the latter case, the assumptions in Lemma \ref{conditional unnecesary discrete} of $\mathcal{Y}$ excluding 0 would not be satisfied.  In such an adversarial case there would be a discrepancy between $\E[|X| \ | \ Y>0]$ and $\E [|X|]$ which would shrink as the proportion of the $N_w^{|\Ji|}N_b$ vectors generated by $\vv$ to which that particular $\hat{\vz}$ is orthogonal.

\begin{lemma} \label{conditional unnecesary discrete}
	Let $X$ and $Y$ be discrete random variables with finite, symmetric support sets $\mathcal{X}$ and $\mathcal{Y}$ respectively, where $0 \notin \mathcal{Y}$, and an even joint probability mass function $f_{XY}(x,y)$ such that $P(X = x, Y=y) = P(X=-x,Y=-y)$. Then
	\begin{align*}
	\E[|X| \ | \ Y>0] = \E [|X|]
	\end{align*}
\end{lemma}

\begin{proof}
	Letting $|X|=Z$, we can make a change of variables to obtain the joint mass function $f_{ZY}(z,y)$, which works out to be 
	\begin{align*}
f_{ZY}(z,y) = \begin{cases}
	f_{XY}(z,y) + f_{XY}(-z, y) & \text{for } \ (z,y) \text{ where } z\in \mathcal{X}^+ \text{ and } y \in \mathcal{Y}\\
	f_{XY}(z,y) & \text{for } (z,y) \text{ where } z=0 \text{ and } \in \mathcal{Y} 
	\end{cases}
	\end{align*}
	where $\mathcal{X}^+$ is the set of all positive elements of $\mathcal{X}$. 
	
	Next, we have that 
	\begin{align}
	\E[Z|Y>0] & = \sum_{z\in \mathcal{X}^+} z P(Z=z | Y>0) \nonumber \\
	& = \sum_{z\in \mathcal{X}^+} z \frac{P(Z=z \cap Y>0)}{P(Y>0)} \label{eq:py0_first}\\
	& = 2 \sum_{z\in \mathcal{X}^+} z P(Z=z \cap Y>0) \label{eq:py0_second}\\
	& = 2 \sum_{z\in \mathcal{X}^+} \sum_{y\in \mathcal{Y}^+} z P(Z=z \cap Y=y) \nonumber\\
	& = 2 \sum_{z\in \mathcal{X}^+} \sum_{y\in \mathcal{Y}^+} z \left(f_{XY}(z,y) + f_{XY}(-z, y) \right) \label{exp of Z conditional disc}
	\end{align}
	
	On the other hand, we have 
	
	\begin{align}
	\E[Z] & = \sum_{z\in \mathcal{X}^+} z P(Z=z)\\
	& = \sum_{z\in \mathcal{X}^+} z\left(f_X(z) + f_X(-z)\right)\\
	& = 2 \sum_{z\in \mathcal{X}^+} zf_X(z)\\
	& = 2 \sum_{z\in \mathcal{X}^+} \sum_{y\in \mathcal{Y}} z f_{XY}(z,y)\\
	& = 2 \sum_{z\in \mathcal{X}^+} \left( \sum_{y\in \mathcal{Y}^+} z f_{XY}(z,y) + \sum_{y\in \mathcal{Y}^-} z f_{XY}(z,y) \right) \label{exp of Z discrete}
	\end{align}
	
	Next, we not that 
	\begin{align*}
	\sum_{y\in \mathcal{Y}^-} z f_{XY}(z,y) & = \sum_{y\in \mathcal{Y}^+} z f_{XY}(z,-y) \\
	& = \sum_{y\in \mathcal{Y}^+} z f_{XY}(-z,y) 
	\end{align*}
	Thus the expressions in \ref{exp of Z conditional disc} and \ref{exp of Z discrete} are equal, which completes the proof.
	
\end{proof}

\begin{lemma} [Expected norm of a random sub-vector]\label{lemma: norm of subvector}
	Let $\vu \in \R^k$ be a fixed vector and let $J \subseteq \{1,2,\dots,k\}$ be a random index set, where the probability of any index from $1$ to $k$ appearing in any given sample is independent and equal to $\alpha$. Then, defining $\vu_J$ to be the vector comprised only of the elements of $\vu$ indexed by $J$, we can lower bound the expectation of the norm of this subvector by
	\begin{align}
	\E_J[\|\vu_J\| ] \geq \alpha\|\vu\|
	\end{align}
\end{lemma}

\begin{proof}
	First, we bound the expectation of the norm in terms of the expectation of the squared norm as follows:
	\begin{align}
	\E[\|\vu_J \|] & = \E[ \sqrt{\sum_{j \in J} u_{J, j}^2} ] \\
	& \geq \frac{1}{\|\vu\|} \E[\sum_{j \in J} u_{J, j}^2 ] \label{ineq: norm of subvector} 
	\end{align}
	This follows because for any $0 \leq \gamma \leq \text{max}(\gamma)$, $\sqrt{\gamma} \geq \frac{1}{\sqrt{max(\gamma)}} \gamma $.
	
	Next we note that $\sum_{j \in J} u_{J, j}^2$ is exactly equivalent to $\sum_{i=1}^k u_i^2 B_i$, a weighted sum of $k$ iid Bernoulli random variables $B_i$ with $p=\alpha$, and so
	\begin{align}
	    \E[\sum_{j \in J} u_{J, j}^2 ] &= \sum_{i=1}^k u_i^2 \cdot \E[B]\\
	    & = \| \vu \|^2 \cdot \alpha.
	\end{align}
	Substituting this into inequality \ref{ineq: norm of subvector} completes the proof,
	\begin{align*}
	\E[\|\vu_J \|] \geq \alpha \|\vu\|
	\end{align*}
\end{proof}

Lemmas \ref{MZ_inequality} and \ref{optimal_MZ_constants} are taken from \cite{ferger2014optimal}, and are restated here for completeness.

\begin{lemma} [Marcinkiewicz-Zygmund Inequality (\cite{ferger2014optimal})]\label{MZ_inequality}
	Let $X_1 , \dots, X_n$ be $n \in \mathbb{N}$ independent and centered real random variables defined on some probability space $(\Omega , A, P)$ with $\E[|Xi |^p ] < \infty$ for every $i \in \{1, . . . , n\}$ and for some $p > 0$. Then for every $p \geq 1$ there exist positive constants $A_p$ and $B_p$ depending only on $p$ such that
	
	\begin{align}
	A_p \E \left[\left( \sum_{i=1}^n X_i^2\right)^{p/2} \right] \leq \E \left[ \bigg \vert \sum_{i=1}^n X_i \bigg \vert^p \right] \leq B_p \E \left[\left( \sum_{i=1}^n X_i^2\right)^{p/2} \right]
	\end{align}
\end{lemma}

\begin{lemma}[Optimal constants for Marcinkiewicz-Zygmund Inequality (\cite{ferger2014optimal})]\label{optimal_MZ_constants}
	Let $\Gamma$ denote the Gamma function and let $p_0$ be the solution of the equation $\Gamma(\frac{p+1}{2}) = \sqrt{\pi}/2$ in the interval $(1,2)$, i.e. $p_0 \approx 1.84742$. Then for every $p>0$ it holds:
	\begin{align}
	A_{p, opt} = \begin{cases}
	2^{p/2 - 1}, & 0 < p\leq p_0\\
	2^{p/2} \cdot \frac{\Gamma \left(\frac{p+1}{2} \right)}{\sqrt{\pi}}, & p_0 < p < 2\\
	1 &  2 \leq p < \infty
	\end{cases}
	\end{align}
	and
	\begin{align}
	B_{p, opt} = \begin{cases}
	1 & 0 < p\leq 2\\
	2^{p/2} \cdot \frac{\Gamma \left(\frac{p+1}{2} \right)}{\sqrt{\pi}}, & 2 < p < \infty
	\end{cases}
	\end{align}
\end{lemma}

\begin{lemma}\label{bound for mod x uniform}
	Let $X = \sum_i \alpha_i w_i$, where $w_i \sim \mathcal{U}(-C,C)$ Then 
	\begin{align*}
	\E[|X|] \geq \frac{C}{2\sqrt{2}} \|\alpha\|
	\end{align*}
\end{lemma}
\begin{proof}
	Defining $X_i = \alpha_i w_i$, we can then apply the Marcinkiewicz-Zygmund inequality with $p=1$, using the optimal $A_1$ from Lemma \ref{optimal_MZ_constants} to get that 
	\begin{align*}
	\E[|X|]  = \E \left[\bigg \vert \sum_{i=1}^{k}X_i \bigg \vert \right]\geq \frac{1}{\sqrt{2}} \E \left[\sqrt{\sum_{i=1}^{k} X_i^2} \right]
	\end{align*}
	Next we use the same tricks as early in the proof of the Gaussian case:
	\begin{align*}
	\frac{1}{\sqrt{2}} \E \left[\sqrt{\sum_{i=1}^{k} X_i^2} \right]  &= \frac{1}{\sqrt{2}} \E \left[\sqrt{\sum_{i=1}^{k} |X_i|^2} \right]\\
	& \geq \frac{1}{\sqrt{2}} \sqrt{\sum_{i=1}^{k} \E[|X_i|]^2},
	\end{align*}
	where the first equality is trivial and the second follows from a repeated application of Jensen's inequality.
	
	To calculate $\E[|X_i|]$ we note that $X_i = \alpha_i w_i$ is uniformly distributed $X_i\sim U(-|\alpha_i|C, |\alpha_i| C)$, and thus 
	\begin{align*}
	\E[|X_i|] = \frac{C |\alpha_i|}{2}
	\end{align*}
	and so 
	\begin{align*}
	\E[|X|] &\geq \frac{1}{\sqrt{2}} \sqrt{\sum_{i=1}^{k} \E[|X_i|]^2} \\
	&= \frac{1}{\sqrt{2}} \sqrt{\frac{C^2}{4}\sum_{i=1}^{k}  |\alpha_i|^2}\\
	& = \frac{C}{2\sqrt{2}} \|\alpha\|
	\end{align*}
	
\end{proof}

\begin{lemma}\label{lemma: bound for mod x uniform discrete}
	Let $X = \sum_i \alpha_i w_i$, where $w_i$ are uniformly sampled from some discrete symmetric sample space $\mathcal{W}$. Then 
	\begin{align*}
	\E[|X|] \geq \frac{\sum_{w\in \mathcal{W}}|w|}{\sqrt{2} N_w} \|\alpha\|
	\end{align*}
\end{lemma}
\begin{proof}
	 Defining $X_i = \alpha_i w_i$, we follow exactly the same steps as in the first part of the proof of Lemma \ref{bound for mod x uniform}, to get that 
	\begin{align*}
	\E[|X|] \geq \frac{1}{\sqrt{2}} \sqrt{\sum_{i=1}^{k} \E[|X_i|]^2}.
	\end{align*}
	
	To calculate $\E[|X_i|]$ we note that $X_i = \alpha_i w_i$ is uniformly sampled from $\alpha_i \mathcal{W}$ and thus 
	\begin{align*}
	\E[|X_i|] = \frac{ |\alpha_i| \sum_{w\in \mathcal{W}}|w|}{N_w}
	\end{align*}
	and so 
	\begin{align*}
	\E[|X|] &\geq \frac{1}{\sqrt{2}} \sqrt{\sum_{i=1}^{k} \E[|X_i|]^2} \\
	&= \frac{1}{\sqrt{2}} \sqrt{\frac{(\sum_{w\in \mathcal{W}}|w|)^2}{N_w^2}\sum_{i=1}^{k}  |\alpha_i|^2}\\
	& = \frac{\sum_{w\in \mathcal{W}}|w|}{\sqrt{2} N_w} \|\alpha\|
	\end{align*}
	
\end{proof}

\begin{lemma}
	Let $\mathcal{W}, \mathcal{X} \subset \R^k$ be discrete sets with finite cardinality, and $g:\mathcal{W} \longrightarrow \mathcal{X}$ be a one-to-one transformation. Then if $P(W=\mathbf{w}) = P(W_1 = w_1, \dots, W_k = w_k) = C$ for all $\mathbf{w}\in \mathcal{W}$, where C is constant, then $P(X = \mathbf{x}) = C$ for all $\mathbf{x} \in \mathcal{X}$
\end{lemma}
\begin{proof}
	\begin{align}
	P(X = \mathbf{x}) &= \sum_{\mathbf{w} \in \{g(\mathbf{w}) = \mathbf{x} \}} P(W = \mathbf{w}) \label{discrete change of vars}\\
	& = C \label{func is one to one}
	\end{align}
	Equation \ref{discrete change of vars} is a change of variables, and \eqref{func is one to one} follows from the fact the there is only ever one term in the sum, since $g$ is one-to-one.
\end{proof}

	

\section{Non-linear dependence on $\alpha$ in the typical case}\label{appendix: alpha dep}

One interesting observation which merits further detail is that the observed dependence of the growth factor on $\alpha$ in practice, shown in \Figref{fig: alpha_dep}, is not exactly linear, but rather the shape of that dependence looks closer to $\sqrt{\alpha}$. The likely source of this qualitative discrepancy is the use of Lemma \ref{lemma: norm of subvector}, to lower bound
\begin{align}
	\E_\Ji[\|dz_\Ji\| ] \geq \alpha\|dz\|,\label{eq: lowerbound subvector}
\end{align}
used in \eqref{eq: lowerbound  subvector in proof} in Stage 3 of the proof of Theorem\ref{thm: general sparse}.
It is straightforward to derive an \textit{upper} bound for this same quantity, as 
\begin{align}
	\E_\Ji[\|dz_\Ji\| ] \leq \sqrt{\alpha}\|dz\|,\label{eq: upperbound subvector}
\end{align}
first using Jensen's inequality to get that $\E_\Ji[\sqrt{\|dz_\Ji\|^2} ] \leq \sqrt{\E[\|dz_\Ji\|^2]}$, and then using the strategy from the proof of Lemma \ref{lemma: norm of subvector} to get $\E[\|dz_\Ji\|^2]=\alpha\|dz\|^2$.

To explore this discrepancy between the observed growth ratio and the lower and upper bounds from \eqref{eq: lowerbound subvector} and \eqref{eq: upperbound subvector}, we consider different fixed vectors $dz\in \R^k$, and average over subvectors $dz_\Ji$. Specifically, we calculated the expected value of a subvector $dz_\Ji$ containing only the entries of $dz$ indexed by $\Ji$, where $\Ji \subseteq \{1,2,\dots,k\}$ is a random index set, where the probability of any index from $1$ to $k$ appearing in any given sample is independent and equal to $\alpha$. \Figref{fig: subvector_of_normal} shows the results when $dz$ a realisation of the uniform distribution over the unit sphere, with different dimensions $k$.

For even moderately large k, and vectors $dz$ where most entries are roughly this same magnitude, this upper bound is very tight, such that the expected norm of the subvector generally behaves like $\sqrt{\alpha}\|dz\|$, not $\alpha\|dz\|$. 
\begin{figure}[!h]
    \centering
\begin{subfigure}[b]{0.45\textwidth}
\includegraphics[width = \textwidth]{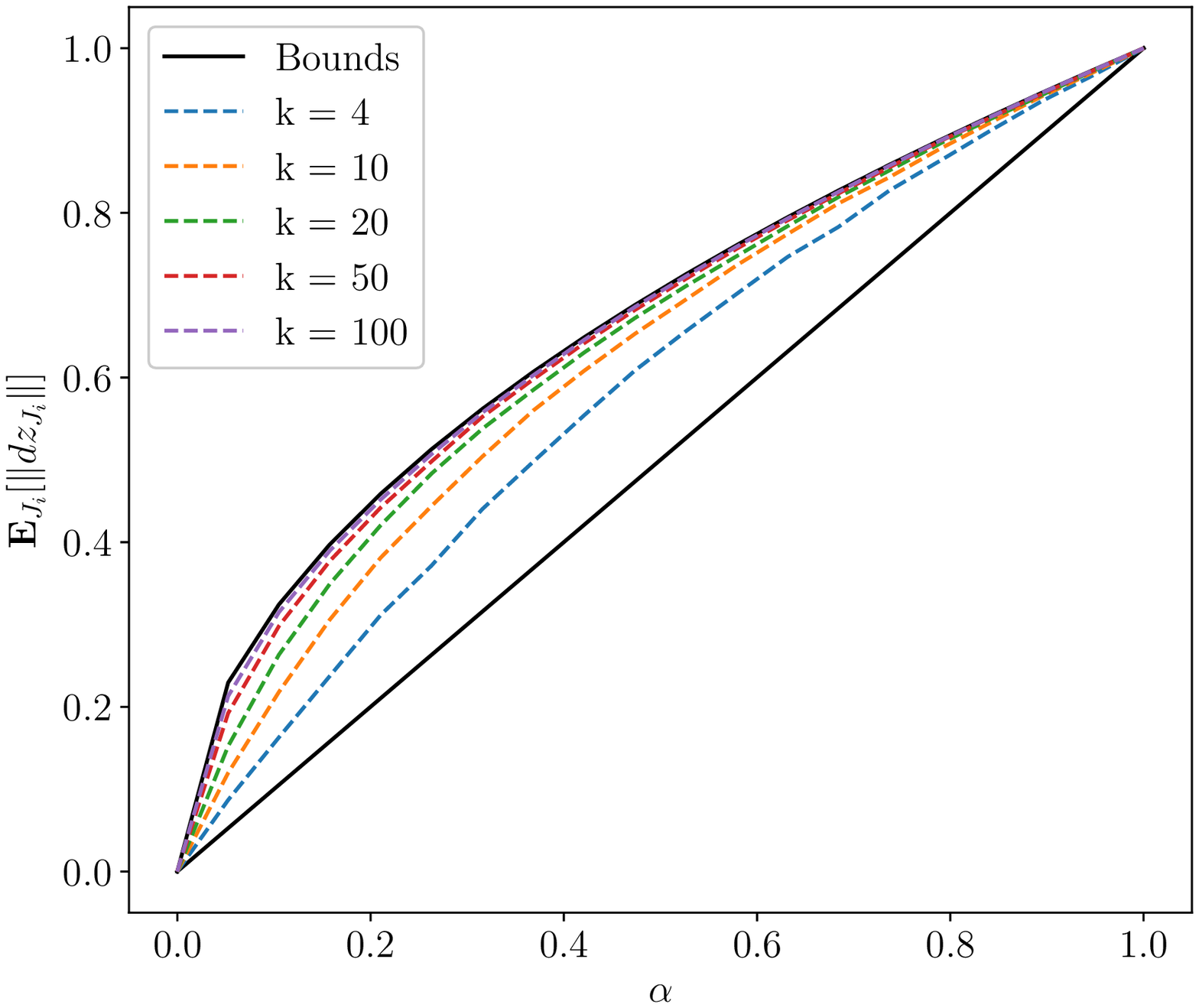}
\caption{}
\label{fig: subvector_of_normal}
\end{subfigure}
\begin{subfigure}[b]{0.45\textwidth}
\includegraphics[width = \textwidth]{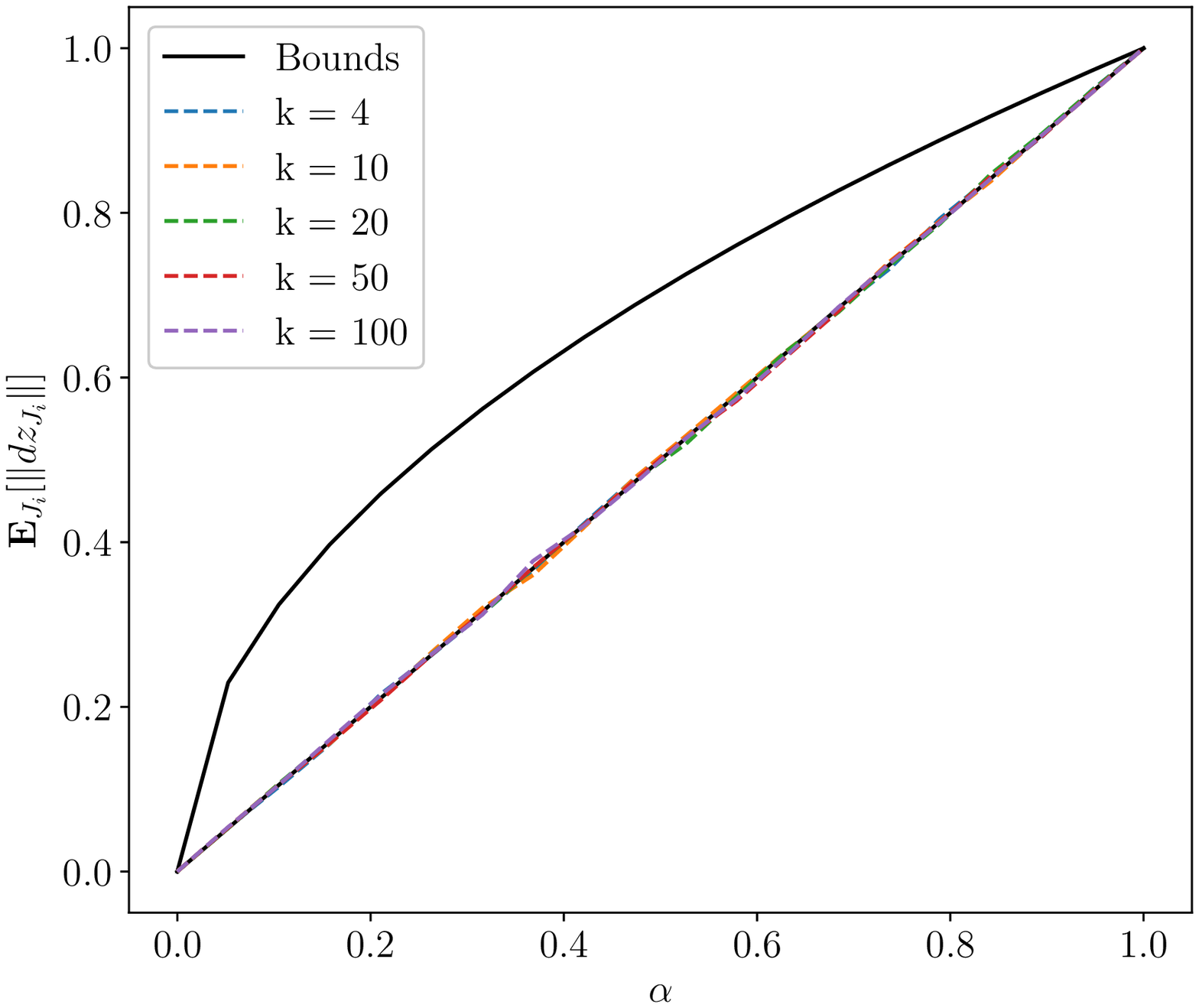}
\caption{}
\label{fig: subvector_of_E1}
\end{subfigure}
\caption{The dependence on $\alpha$ and $k$ of expected value of a subvector $dz_\Ji$. In \Figref{fig: subvector_of_normal}, $dz$ is a realisation of the uniform distribution over the unit sphere. In \Figref{fig: subvector_of_E1}, $dz$ has the first entry equal to 1, and the rest zeros.}
\label{fig: subvectors}
\end{figure}
However,  it is also possible to construct an example where the lower bound is tight, by letting $dz$ have only a single  non-zero entry, which case $\E[\|\vu_J\|] = \alpha \|\vu\|$ (see \Figref{fig: subvector_of_E1}).  While the former  case,  with entries of $dz$ mostly of the same order, is typical, especially past the first few layers of the network, the bound cannot be improved without further assumptions on $\|dz\|$. Further work on quantifying the probabilistic concentration of $\E[\|\vu_J\|]$ close to $ \sqrt{\alpha} \|\vu\|$ would be an interesting extension of this research.

\section{Additional numerical experiments}\label{appendix: additional experiments}

\begin{figure}[!h]
    \centering
\begin{subfigure}[b]{0.45\textwidth}
\includegraphics[width = \textwidth]{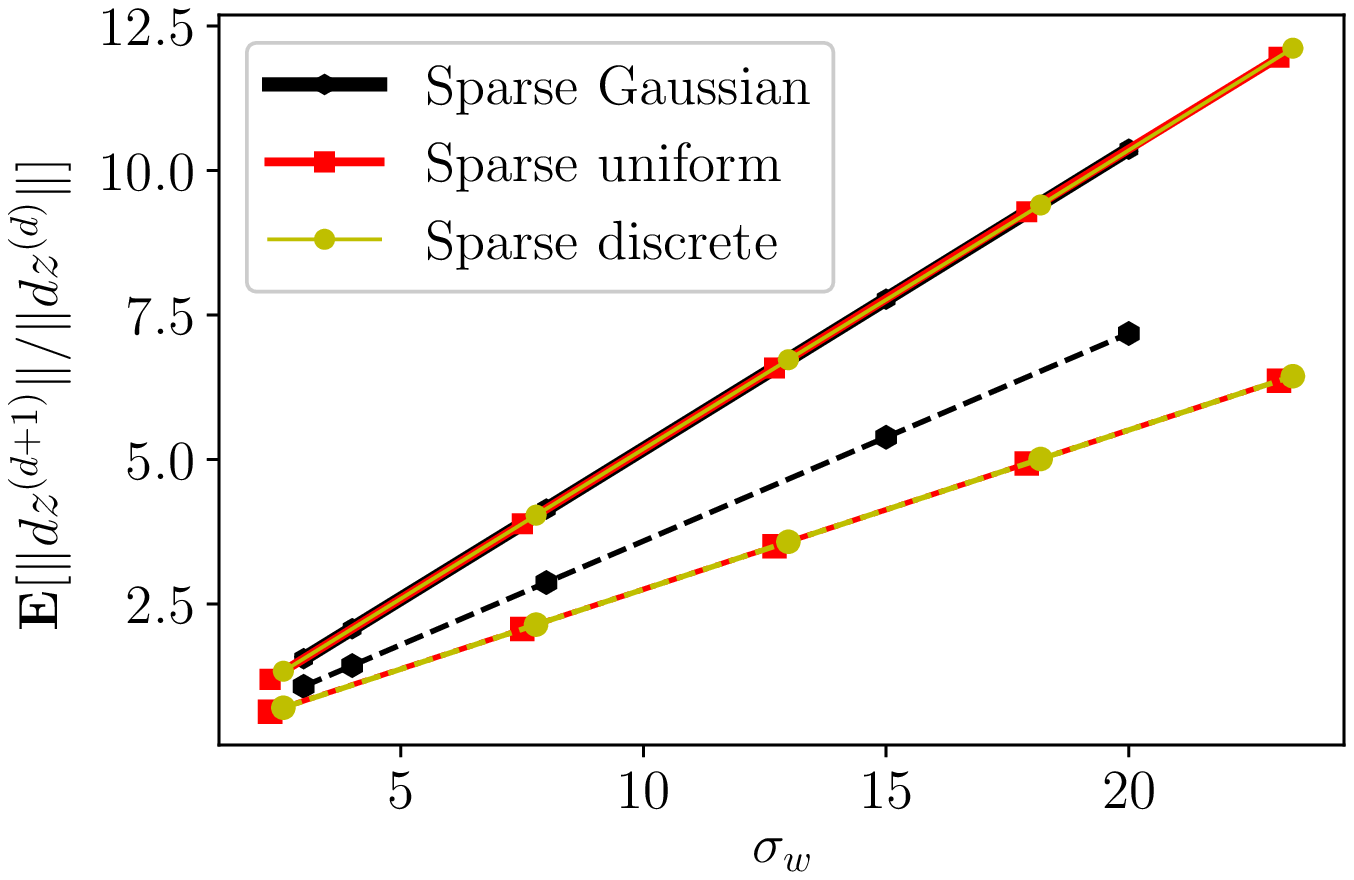}
\caption{}
\label{fig: sig_dep_random_st}
\end{subfigure}
\begin{subfigure}[b]{0.45\textwidth}
\includegraphics[width = \textwidth]{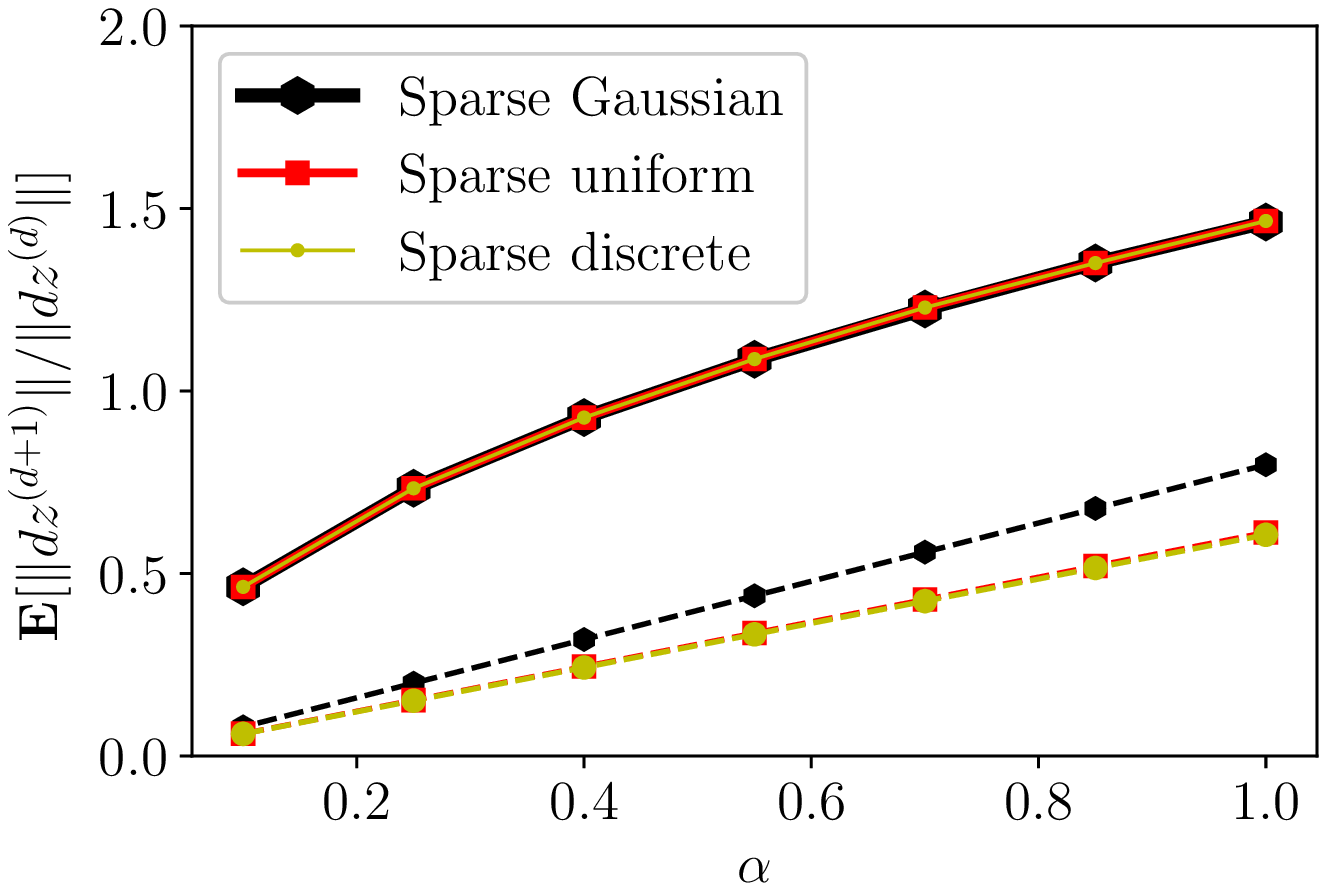}
\caption{}
\label{fig: alpha_dep_random_st}
\end{subfigure}\\
\begin{subfigure}[b]{0.45\textwidth}
\includegraphics[width = \textwidth]{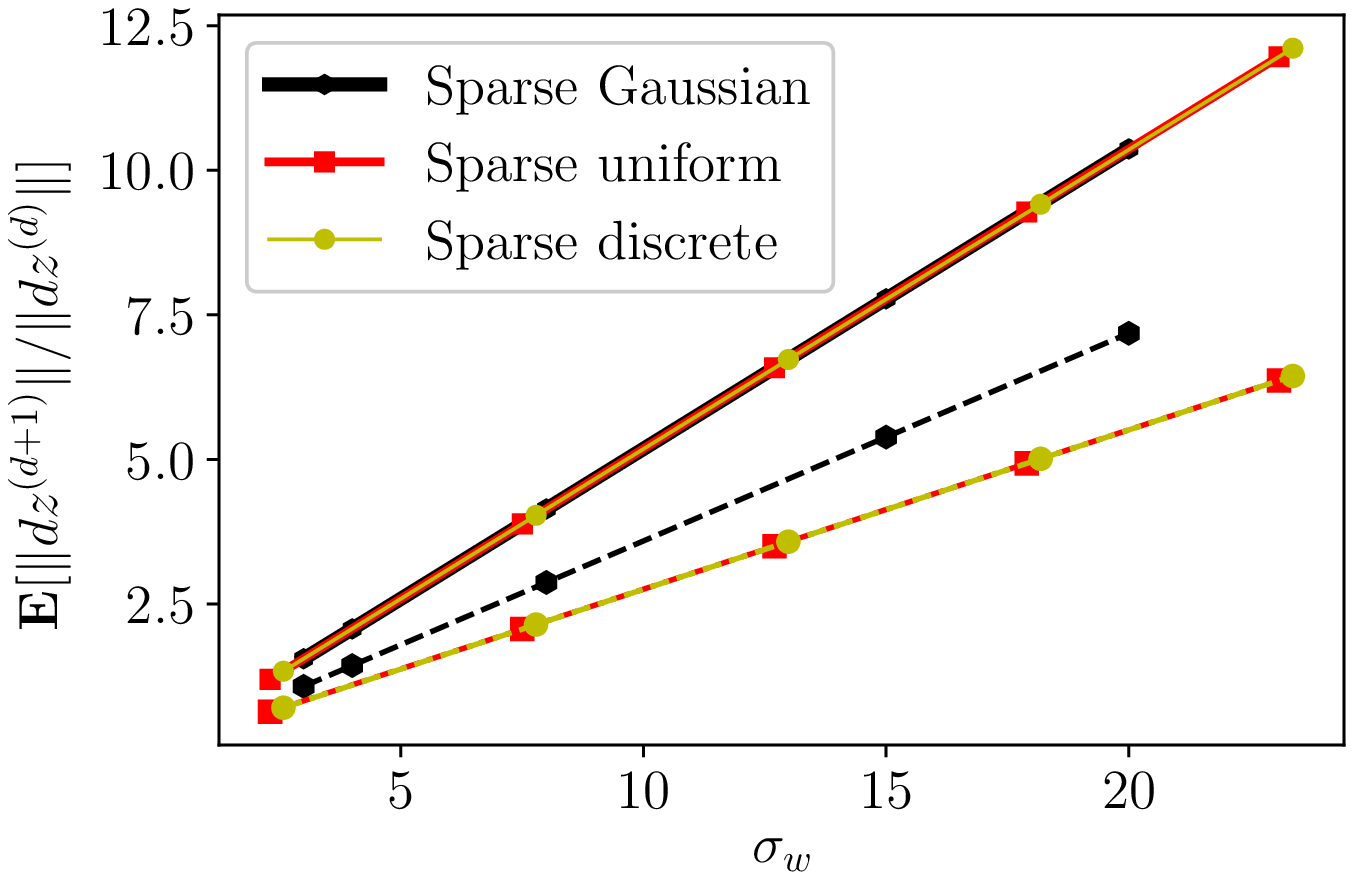}
\caption{}
\label{fig: sig_dep_random_arc}
\end{subfigure}
\begin{subfigure}[b]{0.45\textwidth}
\includegraphics[width = \textwidth]{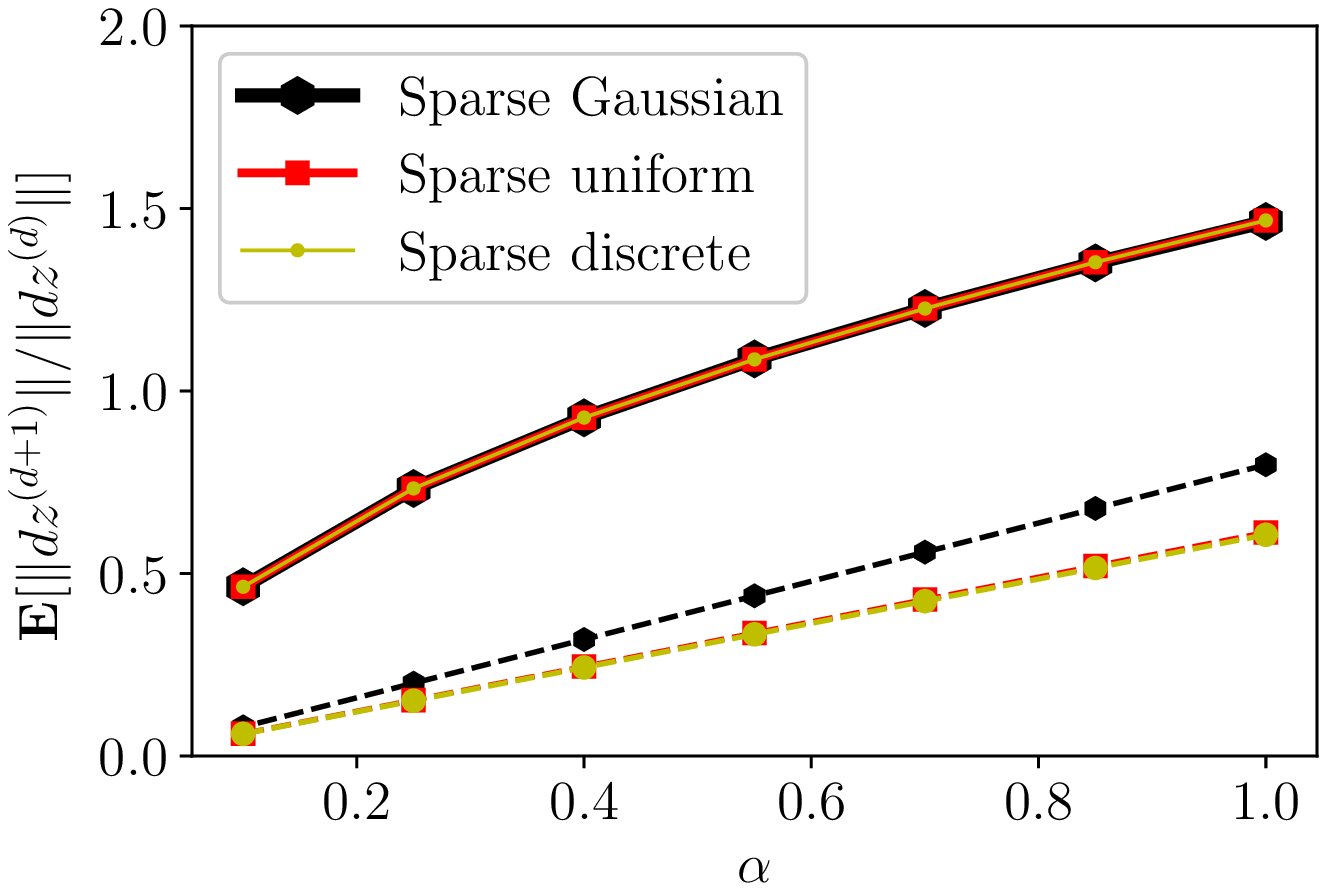}
\caption{}
\label{fig: alpha_dep_random_arc}
\end{subfigure}
\caption{Expected growth factor for trajectories joining randomly chosen (normalised) points in $\R^{500}$. \Figref{fig: sig_dep_random_st} and \Figref{fig: sig_dep_random_arc} show the dependence on the standard deviation of the weights' distribution for a straight and curved trajectory respectively, and \Figref{fig: alpha_dep_random_st} and \Figref{fig: alpha_dep_random_arc} show the dependence on sparsity with a straight and curved trajectory respectively. In this experiment we have chosen as the curved trajectory a straight line which has been modified to be a semi-circular arc in 100 randomly chosen hyperplanes.}
\label{fig: simulations}
\end{figure}

\end{document}